\documentclass[letterpaper, 10 pt, conference]{ieeeconf}  

\IEEEoverridecommandlockouts                              %
\makeatletter
\let\proof\@undefined
\let\endproof\@undefined
\makeatother
\IEEEoverridecommandlockouts
\usepackage{cite}
\usepackage{url}
\usepackage{hyperref}
\usepackage{amsmath,amssymb,amsfonts}
\usepackage{algorithmic}
\usepackage{graphicx}
\usepackage{textcomp}
\let\labelindent\relax
\usepackage{enumitem}
\usepackage{multirow}
\usepackage{booktabs}
\usepackage{xcolor}
\usepackage{flushend}

\usepackage{nicefrac}
\usepackage[font=small]{caption}
\usepackage{subcaption}
\usepackage{lipsum}

\usepackage[ruled,vlined,linesnumbered]{algorithm2e}
\usepackage{algorithmic,algorithm2e,float}

\SetKw{Continue}{continue}
\SetKw{Break}{break}

\newcommand{\vect}[1]{\boldsymbol{#1}}

\newcommand{\be}{\begin{equation}}
\newcommand{\ee}{\end{equation}}

\mathchardef\mhyphen="2D

\newcommand{\dvec}{\vect{D}}

\newcommand{\T}{\mathcal{T}}

\newcommand{\F}{\mathcal{F}}

\newcommand{\Rcal}{\mathcal{R}}

\newcommand{\Greedy}{\mathtt{Greedy}}
\newcommand{\Random}{\mathtt{Random}}
\newcommand{\FLNS}{\mathtt{LNS}}

\newcommand{\delv}{\mathtt{Task\mhyphen removal}}
\newcommand{\delr}{\mathtt{Robot\mhyphen removal}}

\newcommand{\grey}[1]{\textcolor{gray}{#1}}

\usepackage{amsthm}
\theoremstyle{definition}
\newtheorem{problem}{Problem}

\newtheorem*{remark*}{Remark}

\newtheorem{lemma}{Lemma}

\SetKwInput{KwInput}{Input}                
\SetKwInput{KwOutput}{Output}              

\title{\LARGE \bf
Designing Heterogeneous Robot Fleets for\\ Task Allocation and Sequencing
}

\author{Nils Wilde and Javier Alonso-Mora
\thanks{This research is supported by the European Union's Horizon 2020 research and innovation program under Grant 101017008.}
\thanks{N.~Wilde and J.~Alonso-Mora are with the Department for Cognitive Robotics, 3ME,
Delft University of Technology, Delft, Netherlands, 
\texttt{\{n.wilde, j.alonsomora\}@tudelft.nl}.
}%
}

\begin{document}

\maketitle
\thispagestyle{empty}
\pagestyle{empty}

\begin{abstract}
We study the problem of selecting a fleet of robots to service spatially distributed tasks with diverse requirements within time-windows.
The problem of allocating tasks to a fleet of potentially heterogeneous robots and finding an optimal sequence for each robot is known as multi-robot task assignment (MRTA). Most state-of-the-art methods focus on the problem when the fleet of robots is fixed. 
In contrast, we consider that we are given a set of available robot types and requested tasks, and need to assemble a fleet that optimally services the tasks while the cost of the fleet remains under a budget limit.
We characterize the complexity of the problem and provide a Mixed-Integer Linear Program (MILP) formulation. 
Due to poor scalability of the MILP, we propose a heuristic solution based on a Large Neighbourhood Search (LNS).
In simulations, we demonstrate that the proposed method requires substantially lower budgets than a greedy algorithm to service all tasks.
\end{abstract}


\section{Introduction}

Task allocation and sequencing is a fundamental problem in multi-robot systems, with applications in environmental monitoring~\cite{chandarana2021planning, sousa2020decentralized,smith2011persistent}, service in homes and health-care facilities \cite{wilde2022Online,sadeghi2018re}, pickup-and-delivery \cite{mathew2015planning} and autonomous mobility-on-demand \cite{alonso2017demand, zardini2022analysis}.
However, many real world problems pose a diverse set of requirements for robot capabilities, leading to specialized robot designs \cite{zhao2022graph,saberifar2022charting}. Deploying a heterogenous fleet of robots offers the most efficient use of resources, yet poses new challenges due to the increased combinatorial complexity.

In this paper, we consider the problem of designing a heterogenous fleet of robots for complex multi-robot missions. 
Given is a set of available robots with several characteristics such as the capability to service different types of tasks, how they can traverse the environment, their battery life and a deployment cost. The goal is to select robots that can perform their joint mission as well as possible while the total deployment cost remains under a certain budget.
For instance, in material transport applications the most efficient solutions might require a combination of high-capacity robots to service densely located tasks as well as cheaper low-capacity robots reaching remote task locations. Similarly, in environmental monitoring data collection might performed by ground vehicles or drones, which differ in their capabilities of traversing the environment, and potentially in the types of measurements they can take.

We illustrate an example in Figure \ref{fig:intro}. Here a fleet of robots is required to service a set of tasks, each task requires a robot to visit some location in the environment. There are two different types of requirements for the robot to service a task (blue and red). In Figure \ref{fig:intro1}, a fleet of three low-budget robots is deployed, each robot can only service one task type. Due to battery life limits or tight task deadlines, the fleet is not able to complete all tasks. Figure \ref{fig:intro2} shows a different fleet where one red and one blue robot are replaced by a single yet more flexible purple robot that can service both task types. This allows for a different allocation of tasks between robots such that all tasks can be serviced.

Designing a heterogeneous fleet for a mission requires algorithms which efficiently find combinations of different robot types that are well adapted to the tasks.
In  particular, we study the following problem: Given a set of available robot types, assemble a fleet that is able to service the largest number of tasks while the cost of the fleet stays within a budget.
First, we characterize the computational complexity of the problem, and present a mixed-linear integer program (MILP) formulation. 
Due to poor scalability of exact solutions, we propose a heuristic based on a large neighbourhood search (LNS).
LNS algorithms iteratively improve solutions to combinatorial problems by removing elements of the solution and subsequently reinserting these elements. Our method takes an integrated approach of simultaneously optimizing i) which robots to include in the fleet and ii) the missions, \textit{i.e.,} tours, executed by each robot. Therefore, we propose two removal heuristics that allow for alternating between improving the current tours and switching out robots in the fleet. 
In a series of simulation experiments, we show that this approach can find substantially better solutions than a greedy approach for various problem setups. 

\begin{figure}[t]
    \centering
    \begin{subfigure}[t]{0.23\textwidth}
         \centering
        \includegraphics[width=0.99\linewidth]{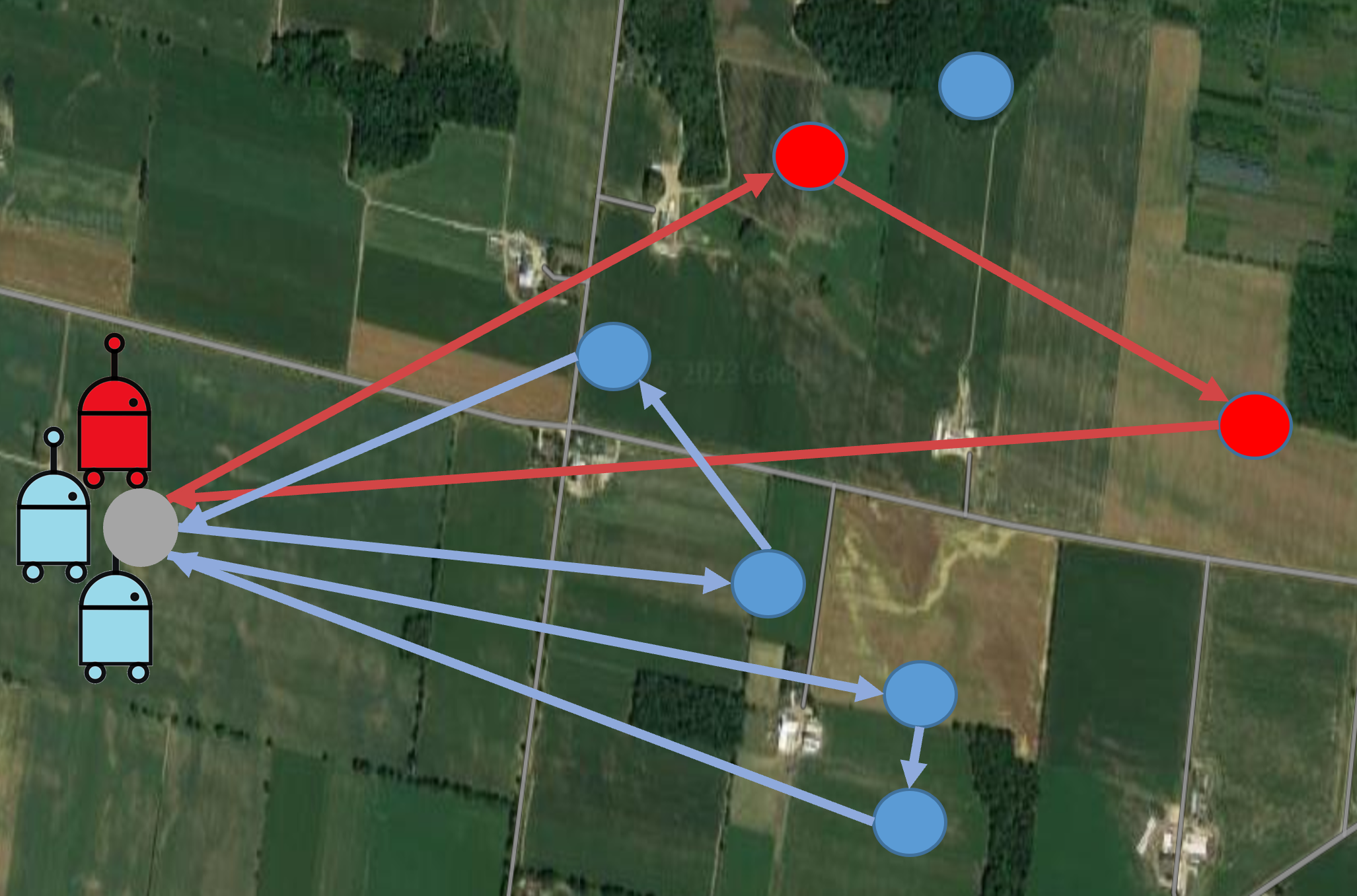}
    \caption{Fleet 1.}
        \label{fig:intro1}
    \end{subfigure}
    \begin{subfigure}[t]{0.23\textwidth}
         \centering
        \includegraphics[width=0.99\linewidth]{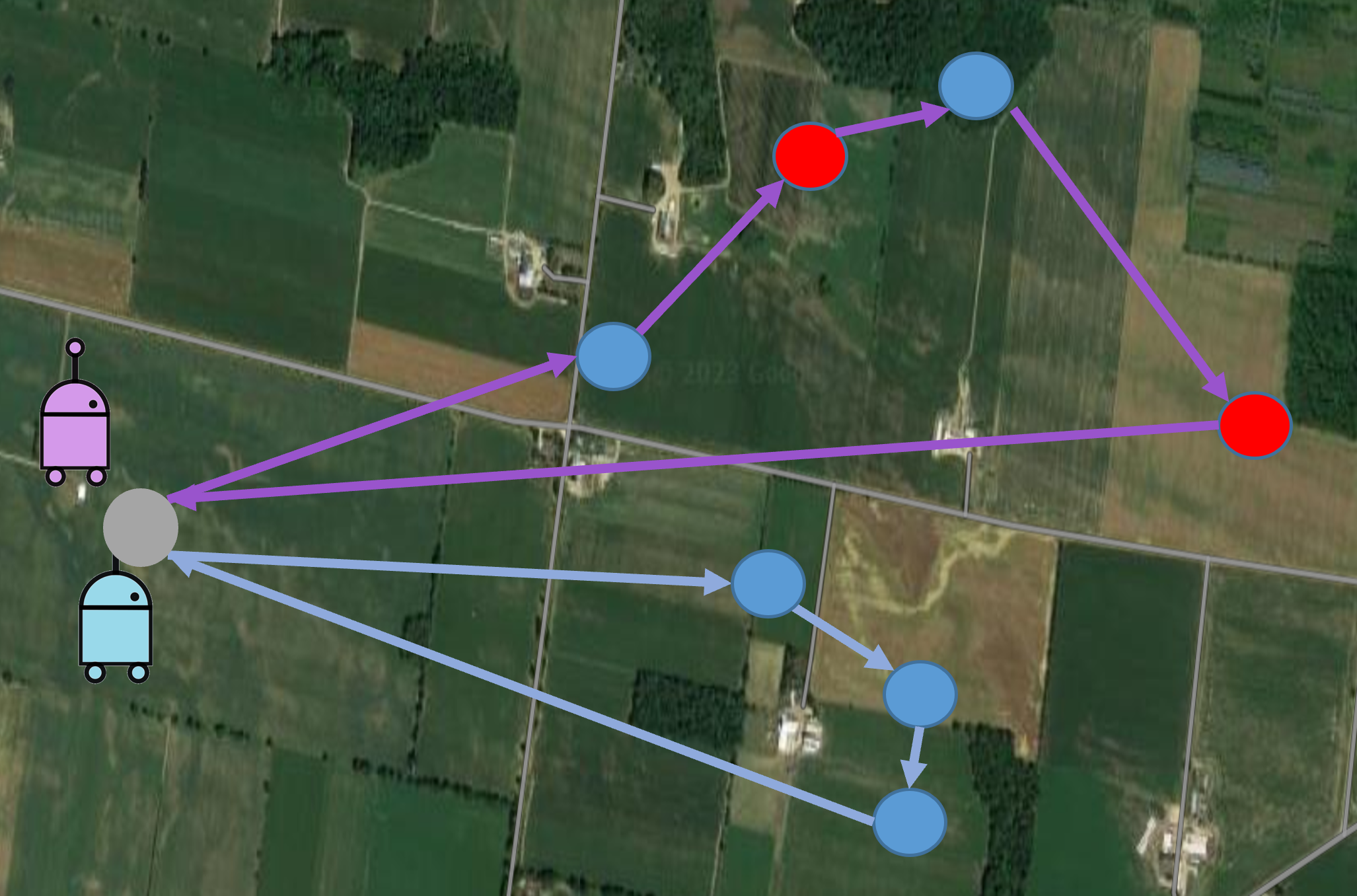}
    \caption{Fleet 2.}
        \label{fig:intro2}
    \end{subfigure}
    \caption{Example for a deployment of heterogenous fleets with different capabilities (red and blue) for different task requirements. Fleet 1 consists of three low cost robots that can each only service one task type. Fleet 2 introduces a more versatile robot (purple) that can service both task types.
    \vspace{-.5cm}
    }
    \label{fig:intro}
\end{figure}

\paragraph{Related Work}
Planning and coordination of heterogeneous robot fleets has found wide interest in the multi-robot systems community in recent years \cite{dixit2019dirichlet,lee2021stochastic,kit2019decentralized,shi2022risk,wesselhoft2022controlling, cecchi2021priority, mathew2015planning, sadeghi2017heterogeneous,xu2020approximation}.
Yet, few works focus on the problem of deciding which robots to include in a heterogeneous fleet. The authors of \cite{zhao2022graph} consider the problem of designing individual underwater robots, and then assembling an optimal fleet. Given the available robots, the fleet design problem is cast into a knapsack variant and solved with a breadth first search. A key difference to our work is that the method is focused on specific mission types and thus imposes different constraints: While we seek to find the best possible fleet given a budget constraint, their work seeks to find a fleet that covers a search area most efficiently.
The work of \cite{rjeb2021sizing} studies the design of homogeneous and heterogeneous robot fleets for logistic warehouses. Both problems are formulated as an Integer Linear Program (ILP). In addition to its focus on pickup-and-delivery, the paper assumes that any type of robot is able to fulfil any task, \textit{i.e.,} transportation request. Our work does not impose this assumption, such that not every fleet is able to service all tasks. Further, we consider additional constraints such as task deadlines and limited battery life.
Related to heterogeneous fleet design, researchers investigate the problem of determining the optimal swarm or fleet sizes for homogeneous multi-robot systems\cite{chandarana2018determining, cap2018multi, chaikovskaia2021sizing, wallar2019optimizing}. 
This either poses a multi-objective optimization problem, trading-off the primary mission objective with the operational cost for the fleet, or a problem of finding the smallest possible fleet such that certain objectives are met. 
Similarly, we are interested in finding a fleet that maximises the number of serviced tasks given a budget limit for assembling the fleet, yet we consider that multiple robots types are available.

Lastly, our solution is based on a large neighbourhood search (LNS). Such methods have been frequently used to solve multi-robot routing problems such as generalized traveling salesman \cite{smith2017glns} or MRTA with homogeneous and heterogeneous fleets \cite{booth2016constraint, sadeghi2017heterogeneous, pan2021multi}.

\paragraph{Contributions}
Our main contributions are as follows. First, we characterize hardness of the fleet design problem and provide an MILP formulation of the problem. Second, we present a large neighbourhood search (LNS) algorithm based on two removal heuristics to simultaneously optimize which robots to be employed in the fleet, and what tour each robot takes.
In simulation, we show the advantage of the proposed LNS approach over a greedy algorithm in several problem variations.

\section{Problem Statement}
We consider an offline multi-robot task assignment problem: Given an environment encoded as a graph $G=(V,E)$, a fleet of robots needs to service a fixed set of $N$ tasks $\T=\{T_1, \dots, T_N\}$ where each task requires one of the robots to travel to a vertex $v\in V$, before some deadline $t^d$.
Further, each task has a set of $\Phi$ requirements. We write a task as a tuple $T=(v,t^d, \Phi)$.

To service tasks $\T$, we can assemble a fleet of robots by choosing from a set of available robot types $\Rcal=\{R_1,\dots, R_l\}$. Each robot type $R\in\Rcal$ is a tuple $(\Psi, b, \beta, \dvec)$. Here $\Psi$ is a set of capabilities to fulfill task requirements; a robot of type $R_i$ can service task $T_j$ when $\Psi_i\supseteq\Phi_j$. Further, $b(R)$ is a fixed deployment cost and $\beta(R)$ the maximum battery life. The matrix $\dvec\in \mathbb{R}_{\geq0}^{|V|\times|V|}$ describes the traversal time of each edge of the graph for a robot of type $R$. This does not only encode different speeds for a robot in different parts of the environment, but can also capture the traversability of edges. For instance, when some edges are only usable by drones, a ground robot would have infinite traversal time for these.
A fleet is a collection of robots $F=\{r_1, r_2, \dots\}$ where each robot $r$ is of some type $R$. A fleet can use multiple robots of the same type, making $F$ a multi-set.
{To service tasks, each robot executes a tour $\tau$, visiting a sequence of task vertices. We do not consider inter-robot collision.}
Given a fleet $F$, an MRTA solver $\pi$ finds a set of tours $Q_{\pi}(F) =\{\tau_1,\tau_2, \dots\}$ for all robots $r_1, r_2, \dots$ in $F$ that optimizes some measure for the quality of service for tasks in $\T$. 
In this paper, we consider the objective to be the number of tasks serviced before their deadline, denoted by the functional $\rho(Q_{\pi}(F),\T)$. Finally, the length of each tour $l(\tau_i)$ must remain below the robot's battery limit $\beta(r_i)$. 
The fleet design problem is then formulated as follows:
\begin{problem}[Budgeted fleet design]
\label{prob:fleet_design}
Given an environment described by $G=(V,E)$, a set of tasks $\T$, an infinite supply of different robot types $\Rcal$ and a budget $B$, find a fleet of robots $F^*$ and an MRTA solver $\pi$, that solve
\be
\label{eq:objective}
\begin{aligned}
\max_{F} \;&\rho(Q_{\pi}(F), \T)\\
s.t.&\,  \sum_{r_i\in F} b(r_i) \leq B\\
&\,l(\tau_i) \leq \beta(r_i), \text{ for all }i=1,\dots,|F|\\
&\, r_i \in \Rcal \text{ for all } r_i \in F.\\
\end{aligned}
\ee
\end{problem}

In essence, the problem at hand seeks to find the best fleet given the available robot types to service a set of tasks, while the budget for robots is limited.

\section{Method}

First, we formulate the problem as an optimization of a set function in order to relate it to well-known combinatorial optimization problems, establishing hardness results, and to present a mixed integer linear program.
We then briefly study the greedy algorithm before presenting our proposed method based on a Large Neighbourhood Search (LNS).

\subsection{Approach}
We approach Problem \ref{prob:fleet_design} by formulating it as a subset selection problem. Given the robot types, let $\F$ be a multi-set containing $\lceil\nicefrac{B}{b_i}\rceil$ robots of each type $R_i$. By construction, any feasible solution to \eqref{eq:objective} is a subset of $\F$. 
This allows us to treat the problem as a heterogeneous variant of the team orienteering problem (TOP) \cite{vansteenwegen2011orienteering, Wilde2021LearningSubmod} with robots $\F$ where only some robots execute a tour of non-zero length, \textit{i.e.} \emph{activated}.
Thus, let $Q = \{\tau_1, \tau_2, \dots\}$ be the set of tours for all robots $r_1, r_2,\dots$ in the base set $\F$.
Further, let $q_i$ be an indicator taking value $1$ if $\tau_i\neq\emptyset$ and $0$ otherwise.
Given a budget $B$, tours $Q$ are \textit{feasible} when the cost of all robots with non-empty tours do not exceed the budget, \textit{i.e.,}
\be
\sum_{i=1}^{|Q|} q_i b(r_i)\leq B.
\label{eq:budget_feasible}
\ee

\paragraph{Computational Hardness}

It might not be surprising that Problem \ref{prob:fleet_design} is intractable since it contains MRTA as a subproblem. 
Nonetheless, we briefly study the hardness of the problem in more detail to highlight two important details.

We begin with stating the decision version of the problem: \textit{Given the inputs of Problem \ref{prob:fleet_design}, does there exist a fleet such that the reward jointly collected by the optimal tours for each robot is larger or equal to some given constant?}
We make two observations: First, in order to be a member of NP, a certificate for Problem \ref{prob:fleet_design} consists of not only a fleet but also a set of tours for the robots. Otherwise, the correctness of a solution could not be verified without solving the underlying vehicle routing problem.
Second, the problem is also NP-hard when the vehicle routing part is trivial. To show this, we provide a reduction from the $0/1$ knapsack problem.

\begin{lemma}[NP-hardness]
    The decision version of Problem \ref{prob:fleet_design} is NP-hard.
\end{lemma}
\begin{proof}
    We prove hardness via reduction from a $0/1$ knapsack problem \cite{kellerer2004multidimensional}.
    Given an instance of Knapsack consisting of a set of $n$ items, their weights and profits and a budget (all integer values), we construct a graph with a central depot and exactly one vertex for each item. For each vertex, we create several identical tasks with a deadline $\geq1$ and a requirement unique to the item. The number of task copies for each item equals the profit of the item. Further, let there be $n$ robot types $r_1,\dots r_n$ with disjoint capabilities, each able to service exactly one task. The robot costs $b(r_i)$ equal the weight of the $i$-th item.
    Finally, all robot travel times equal $1$, and the battery life $\beta=2$.
    Thus, each robot can visit exactly one vertex, service all task copies located there, and return to the depot. Finally, the budget $B$ equals the budget of the knapsack.
    A solution to the fleet design problem then contains a set of of tours for robots. This can easily be converted into a solution to the Knapsack problem: Looping over the tours $Q$, we identify the tasks being serviced, which directly correspond to items for the knapsack problem.
\end{proof}
We observe that vehicle routing is trivial in the instance created in the reduction: Each robot can only visit exactly one vertex and then returns to the depot.
This highlights that the fleet design problem is computationally intractable, regardless of the hardness of the underlying MRTA problem.


\paragraph{MILP Formulation}

We present an MILP formulation, adapting notation from MILPs for the TOP with time windows \cite{lin2017solving}. Given the inputs of the problem we can construct a complete graph $\bar{G}$ where vertices correspond to all $N$ task locations and the central depot.
For each robot type, the distance matrix $\bar{\vect{D}}$ then describes the shortest distances on the original graph $G$.
First, we create a copy of the depot vertex where each robot's tour will end such that the graph has vertices $0,1,\dots, N+1$.
Indices $i,j$ refer to vertices, index $k$ identifies a robot ranging from $1$ to $K$ where $K=|\F|$.
We use three binary decision variables. i) $x_{ij}^k=1$ indicates that robot $k$ traverses edge $(i,j)$, ii) $y_{i}^k=1$ indicates robot $k$ service task $i$, and iii) $z^k=1$ indicates if robot $k$ is used. We use two auxiliary variables: $s_i^k$ denotes the time robot $k$ visits vertex $i$ and $c_i^k$ is a binary variable indicating if robot $k$ can service the task $i$. Lastly, let $M$ be some large constant.
The MILP is then given by:

\begin{subequations}
\begin{align}
&\max  \sum_{i=1}^N\sum_{k=1}^K y_i^k \\
s.t.&
\sum_{j=1}^{N} x_{0j}^k 
= \sum_{j=1}^{N} x_{j,N+1}^k 
=z^k
\\&
\sum_{i=0, i\neq l}^{N+1} x_{il}^k =
\sum_{j=0, j\neq l}^{N+1} x_{lj}^k 
= y_l^k 
\\&
s_i^k + \bar{D}_{ij}^k - s_j^k \leq M(1-x_{ij}^k) 
\\&
\sum_{k=1}^K y_i^k\leq 1 
\\&
y_i^k \leq c_i^k  
\\&
s_{i}^k\leq y_i^k \cdot t^d_i 
\\&
\sum_{i=0}^N+1\bigg(
\sum_{j=0}^N+1
x_{ij}^k \bar{D}_{ij}^k 
\bigg)\leq \beta^k 
\\&
\sum_{k=1}^K z^k b^k \leq B 
\\&
s_{i}^k\geq0 
\\&
x_{ij}^k, y_i^k \in \{0,1\} \; \forall i,j=0,\dots, N+1;\ k=1,\dots,K.
\end{align}
\end{subequations}
The objective (3a) counts the number of serviced tasks. (3b) ensures that only activated robots leave the depot. (3c) and (3d) establish connectivity and record the time a robot visits a task. (3e-g) ensure that only one robot services each task, a robot has the capabilities to service a task, and the task is serviced before the deadline. (3h) and (3i) enforce the battery life and budget constraints. Finally, (3j) and (3k) ensure positive arrival times and initialize the variables.

\subsection{Greedy Algorithm}
\label{sec:greedy}
To the best of our knowledge there are no state-of-the-art methods for Problem \ref{prob:fleet_design}. Thus, we introduce a greedy algorithm to establish a baseline.
Beginning with an empty fleet $F=\emptyset$, $\Greedy$ iteratively adds the robot $r^*$ that solves:
\be
\label{eq:greedy_step}
r^* = \arg\max_{r_i\in R}
\frac{\rho(Q_{\pi}(F\cup\{r_i\}), \T) - \rho(Q_{\pi}(F), \T)}
{b_i}.
\ee
The process is repeated until the budget is exhausted.
However, MRTA itself is NP-hard and thus the greedy step cannot be solved within polynomial time. Moreover, approximation algorithms for MRTA are often not available under complex constraints such as task deadlines.
Thus, $\Greedy$ does not have an approximation guarantee.

\subsection{Fleet Optimization via Large Neighbourhood Search}
We now present our large neighbourhood search (LNS) approach to Problem \ref{prob:fleet_design}.
Beginning with an arbitrary initial solution, LNS algorithms repeatedly remove parts of the solution, \textit{e.g.,} vertices from a TSP tour, and then reinsert these elements. 
This allows LNS approaches to solve large instances of complex problems.

\paragraph{Algorithm Overview}
We propose an integrated LNS algorithm that simultaneously optimizes which robots are part of the fleet and what route each robot takes.
Thus, our method uses two removal heuristics, allowing it to randomly switch between a) changing which robots are currently part of the fleet and b) improving the current tours.
The main procedure is summarized in Algorithm \ref{alg:FLNS}.
We begin with a randomly generated feasible solution $Q$. Over $K$ iterations, we select a mode for the removal heuristic and then remove a subset of all tasks from the tours $Q$ (line 6). 
We then re-enter removed tasks $\T'$ into the subtours $Q'$ (line 7), and update the current solution $Q$ and best solution found thus far $Q^{\text{best}}$(lines 8-9).
Further, we use simulated annealing and potentially accept a suboptimal new solution to allow for more exploration in early iterations of the algorithm (lines 11-12).
Next, we will present the two proposed removal heuristics and the insertion heuristic in detail.

\begin{algorithm}[t]	
	\DontPrintSemicolon 
	\KwIn{Graph $G$, depot $s$, tasks $\T$, robot classes $R$, budget $B$, integer $K$.}
	\KwOut{Robot fleet $F$ and tours $Q$}
    Initialize tours $Q = \{\emptyset, \emptyset, \dots\}$ for all $r_i\in \F$\\
    
    $Q \leftarrow$ Randomly generate feasible tours\\
    $Q^{\text{best}}\leftarrow Q$\\
    \For{$k=1$ to $K$}{
        $\mathtt{mode}\leftarrow \mathtt{Select\_Removal\_Mode}$ \grey{ // tasks or robots}\\  
        $Q', \T'\leftarrow$ $\mathtt{Removal}(Q,\mathtt{mode})$\\
        $Q^{\text{new}} \leftarrow \mathtt{Repair}(Q', G, \T', B)$\\
        \If{$\rho(Q^{\text{new}}, \T) > \rho(Q, \T)$}{
        $Q^{\text{best}}\leftarrow Q^{\text{new}}$\\
        $Q\leftarrow Q^{\text{new}}$
        }
        \If{$\mathtt{Accept}(Q^{\text{new}}, k)$}
        {$Q\leftarrow Q^{\text{new}}$}
    }
    $F\leftarrow$ robots with non-empty tours in $Q^{\text{best}}$\\
	\Return{$F$, $Q^{\text{best}}$}
	\caption{Fleet LNS}
	\label{alg:FLNS}
\end{algorithm}

\paragraph{Removal Heuristics}
We propose two removal heuristics: $\delr$ and $\delv$ to enable replacing robots in the fleet and improving current tours.

$\delr$ selects a random subset of all robots and deletes their tours entirely, effectively removing these robots from the fleet. 
A tuning parameter  $n^{\mathtt{R}}$ defines an upper limit for the size of the subset of robots that are removed.
This allows for replacing parts of the fleet with different robots in the subsequent $\mathtt{Repair}$ step.

$\delv$ removes a random subset of the tasks of each current tour. Similar to the first heuristic, a parameter $n^{\mathtt{T}}$ describes the maximum percent of tasks that can be deleted.
The subsequent $\mathtt{Repair}$ step then may improve the tours without changing which robots are \textit{active} (\textit{i.e.,} in the fleet).
Finally, the function $\mathtt{Select\_Removal\_Mode}$ randomly chooses one of the heuristics following some bias $p^{\mathtt{removal}}$.

\paragraph{Insertion Heuristic}

\begin{algorithm}[t]	
	\DontPrintSemicolon 
	\KwIn{Graph $G$, current tours $Q$, tasks $\T$, unassigned tasks $\T'$, budegt $B$.}
	\KwOut{New tours $Q'$}
    $Q'\leftarrow Q$\\
    \While{$\T'$ is not empty}{
    $T\leftarrow \mathtt{pop\_random\_task}(\T',Q')$\\    
     \For{$r_i$ in $\F$}{
        $\tau_i'\leftarrow$ best insertion of $T$ in tour $\tau_i$\\  
        $Q^i\leftarrow Q'\cup \tau'_i\setminus\tau_i $\\ 
        Compute utility $z^i$\\       
        }
    \If{$\max_i \ z^i > 0$}{
       $Q'\leftarrow$ feasible $Q^i$ with largest $z^i$
    }
    
    }

	\Return{$Q'$}
	\caption{$\mathtt{Repair}$ (Insertion heuristic)}
	\label{alg:insertion}
\end{algorithm}
We now present our $\mathtt{Repair}$ heuristic.
Consider a robot $r_i$ with current tour $\tau_i$, and some task $T$.
A maximum reward insertion then finds position in the tour $\tau_i$ such that the reward of adding $T$ at that position is maximized. We notice that this is independent of the other tours.
Let $Q$ be the current set of tours, and let $Q^i$ denote the set of tours when $T$ is added to robot $r_i$. 
One could then assign $T$ to the robot $r_i$ where the marginal gain $\rho(Q^i,\T)-\rho(Q,\T)$ is largest. However, such a strategy does not consider the cost of adding a new robot to the fleet.
Thus, we construct a utility function.
First, we check if $Q^i$ is feasible with respect to the budget $B$ (equation \eqref{eq:budget_feasible}) and each robots battery-life constraint. If $Q^i$ is infeasible, the utility is zero. 
If the insertion is feasible, we compose the utility using the marginal gain, a discount factor $\delta^i$ and a noise term $\eta$:
\be
z^i =
(1+\eta)
        \; \delta^i
        \;(\rho(Q^i, \T)-\rho(Q, \T)).
\ee
The discount factor penalizes adding $T$ to a robot which currently has an empty tour. Thus, $\delta^i=1$ if $\tau_i$ is not empty. 
Otherwise, $\delta^i$ is drawn randomly from $\{1,\nicefrac{1}{\beta(r_i)}\}$ with some probability $p^{\mathtt{discount}}$, \textit{i.e.,} randomly discounts the marginal gain with the robots deployment cost.
Lastly, $\eta=\in[0,.1]$ is a small uniformly random noise term, increasing exploration.

Given currently unassigned tasks $\T'$, we select a task randomly (line 3) and then find the maximum reward insertion into the current tour $\tau_i$ for every robot $r_i$ in the base fleet $\F$ (lines 4-6). 
We compute a utility $z^i$ for the insertion $Q^i$ (lines 7) and update the new set of tours $Q'$ (lines 8-9). This is repeated until the set of unassigned tasks $\T'$ is empty.

\begin{figure}[t]
    \centering
    \begin{subfigure}[t]{0.2\textwidth}
         \centering
        \includegraphics[width=0.8\linewidth]{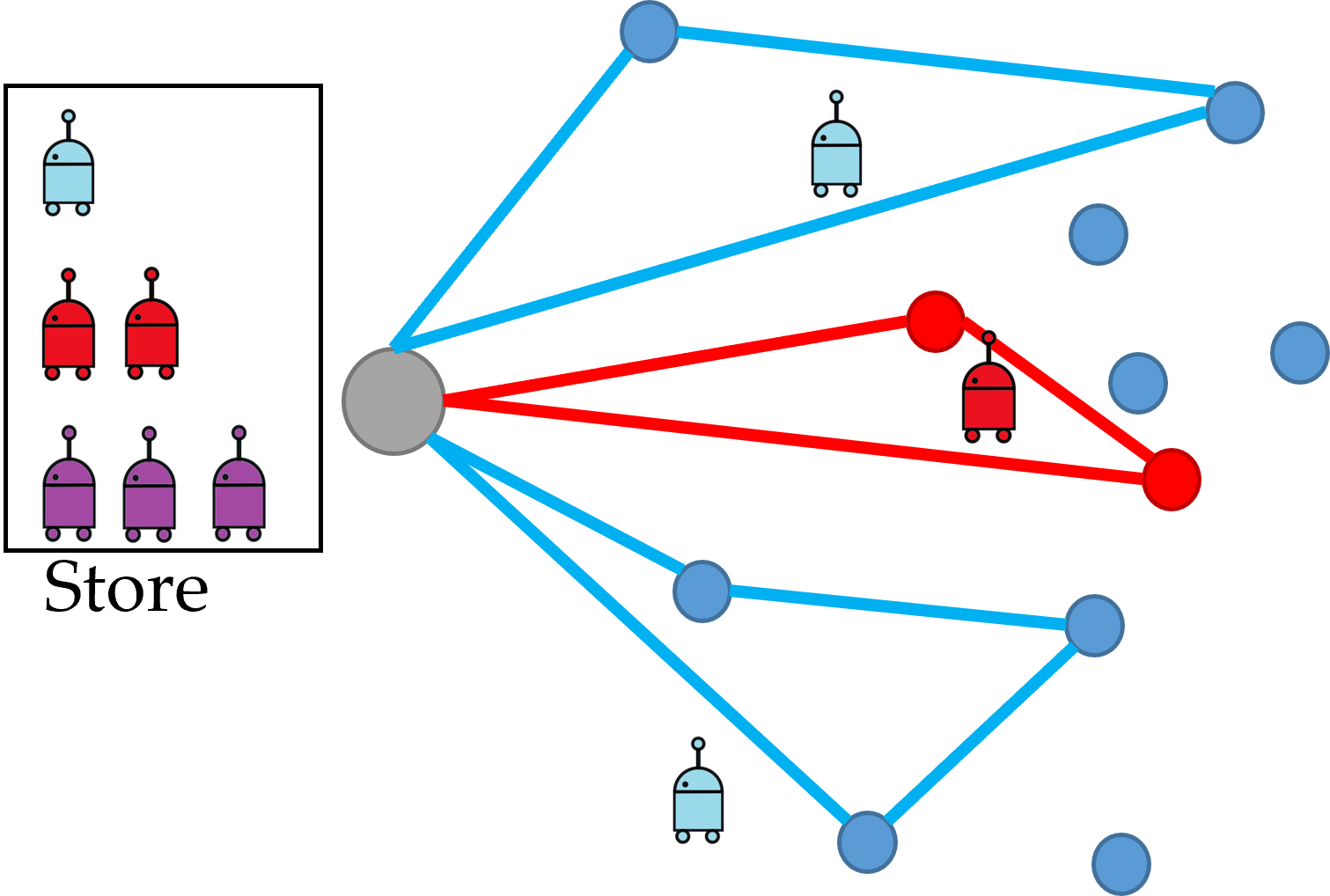}
    \caption{Initial solution.}
        \label{fig:alg_initial}
    \end{subfigure}\\
    \vspace{.5cm}
    \begin{subfigure}[t]{0.2\textwidth}
         \centering
        \includegraphics[width=0.8\linewidth]{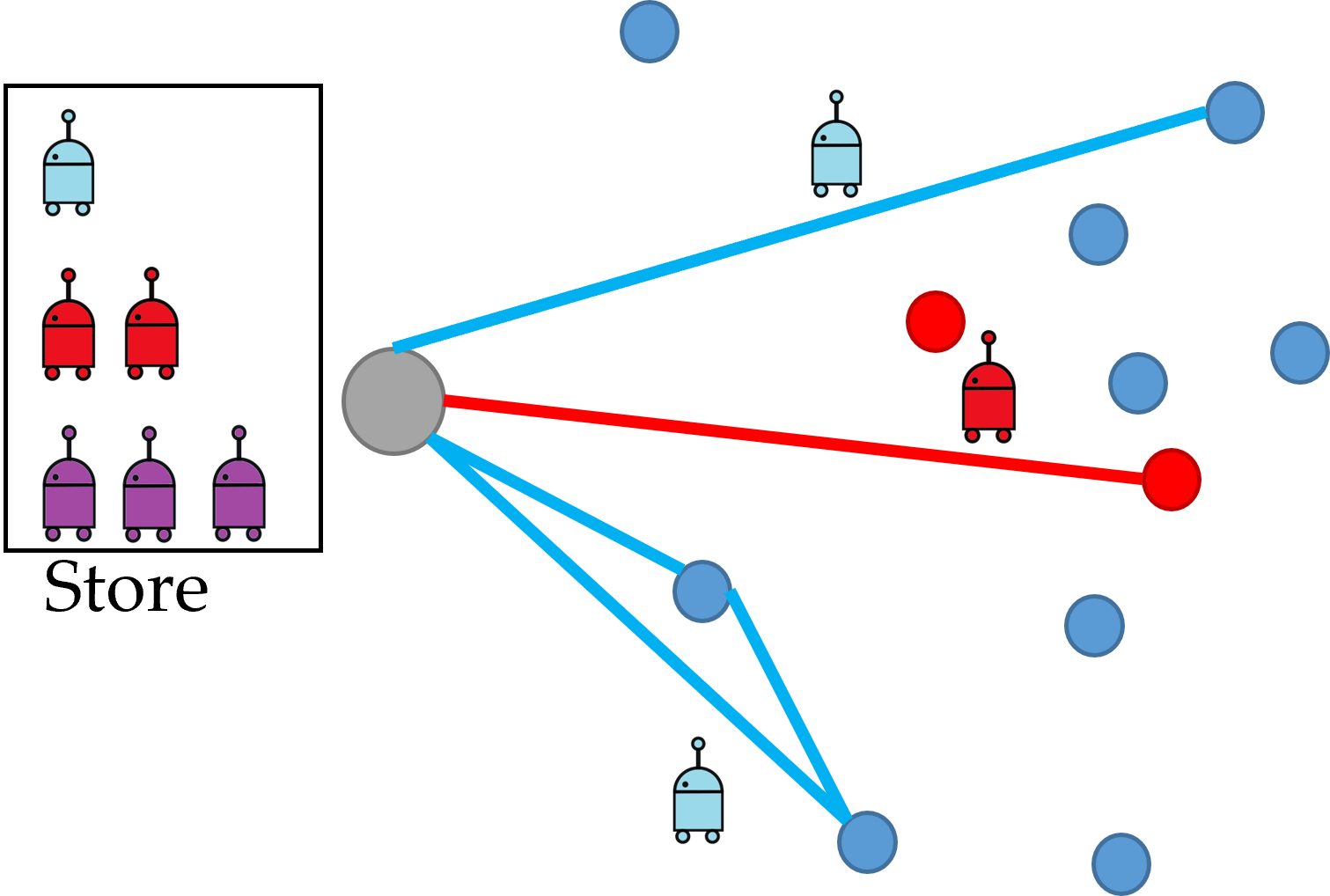}
    \caption{Iteration 1: $\delv$.}
        \label{fig:alg_iter1remove}
    \end{subfigure}
    \hspace{.5cm}
    \begin{subfigure}[t]{0.2\textwidth}
         \centering
        \includegraphics[width=0.8\linewidth]{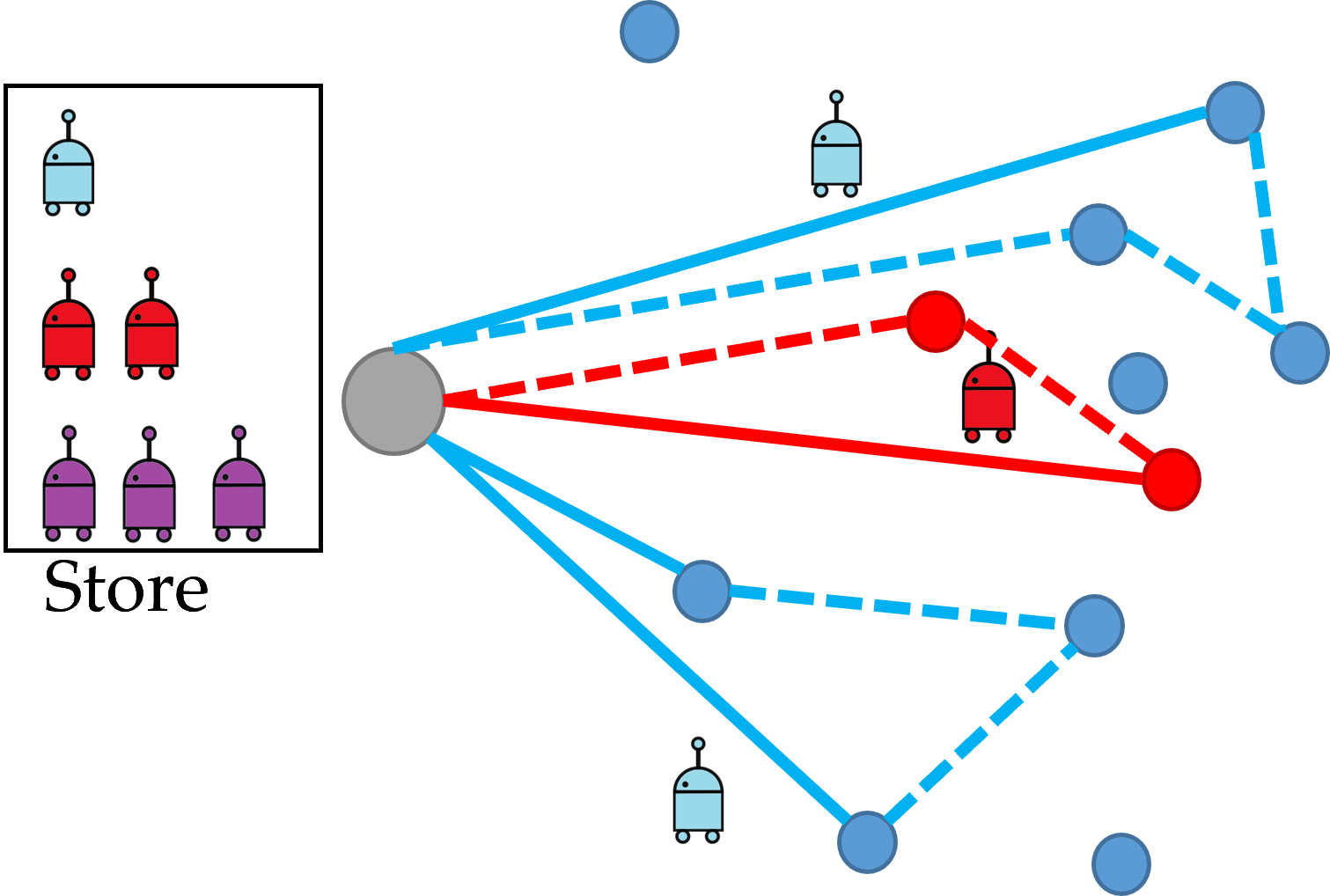}
    \caption{Iteration 1: $\mathtt{Repair}$.}
        \label{fig:alg_iter1repair}
    \end{subfigure}\\
    \vspace{.5cm}
    \begin{subfigure}[t]{0.2\textwidth}
         \centering
        \includegraphics[width=0.8\linewidth]{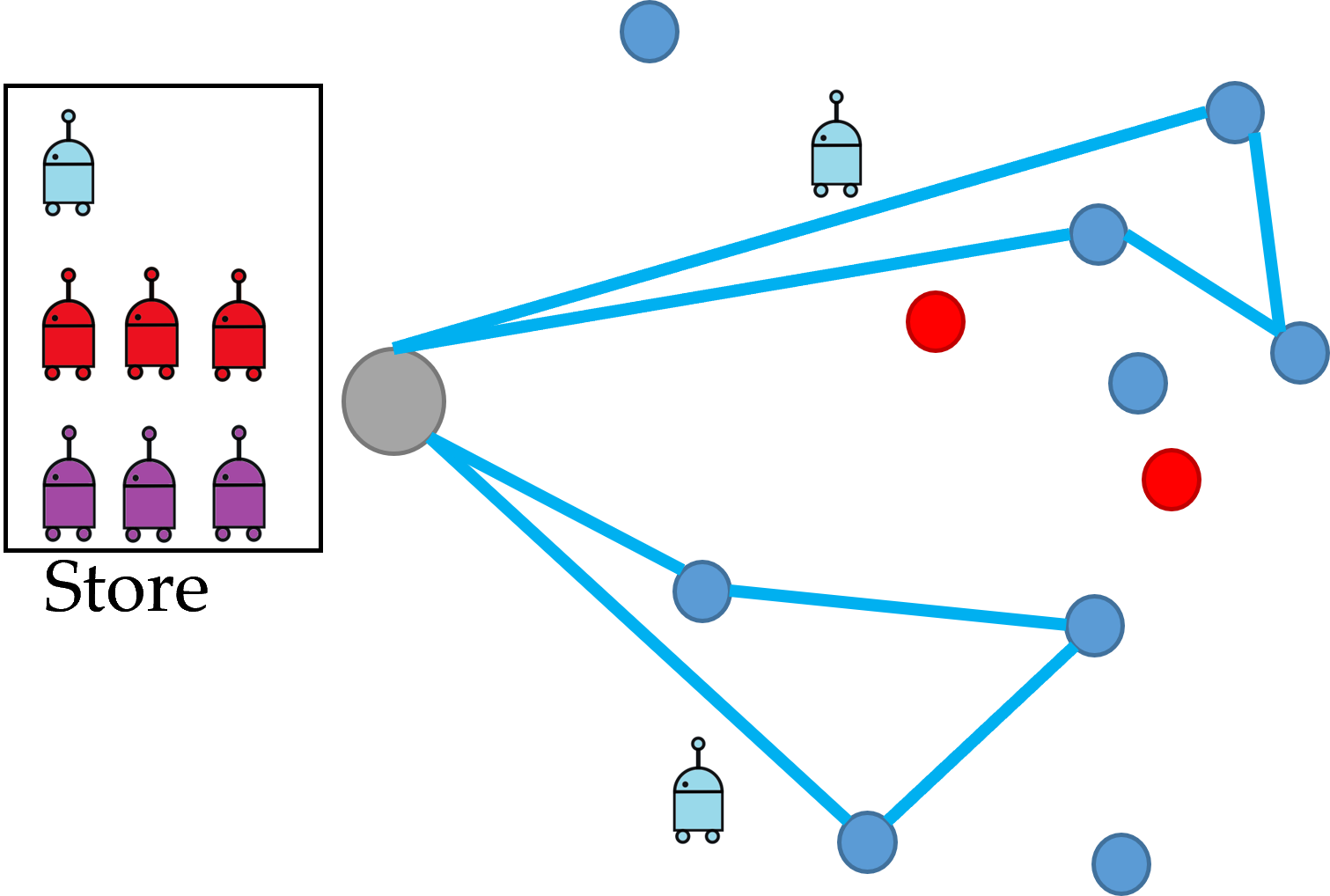}
    \caption{Iteration 2: $\delr$.}
        \label{fig:alg_iter2remove}
    \end{subfigure}
    \hspace{.5cm}
    \begin{subfigure}[t]{0.2\textwidth}
         \centering
        \includegraphics[width=0.8\linewidth]{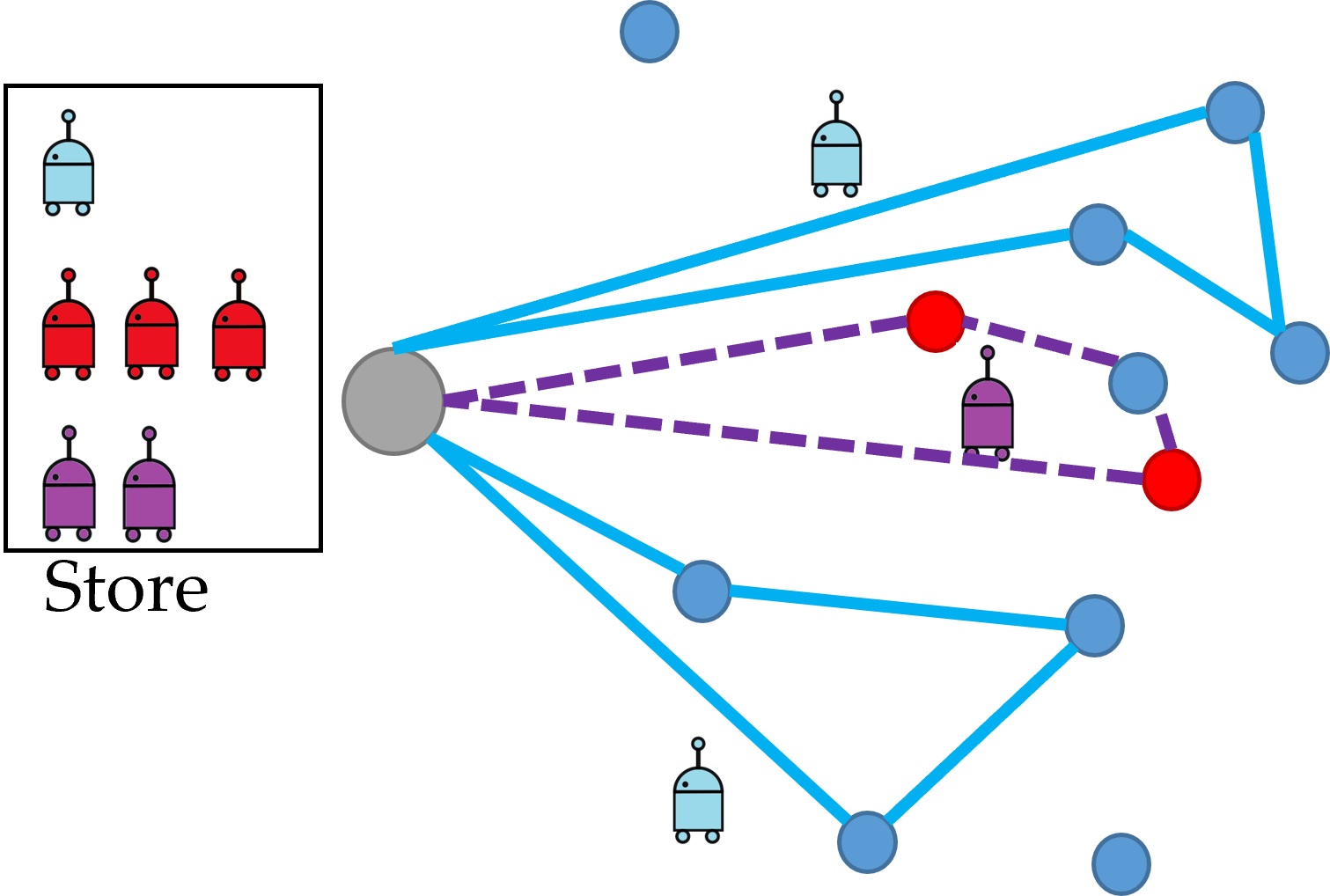}
    \caption{Iteration 2: $\mathtt{Repair}$.}
        \label{fig:alg_iter2repair}
    \end{subfigure}\\
    \vspace{.5cm}
    \begin{subfigure}[t]{0.2\textwidth}
         \centering
        \includegraphics[width=0.8\linewidth]{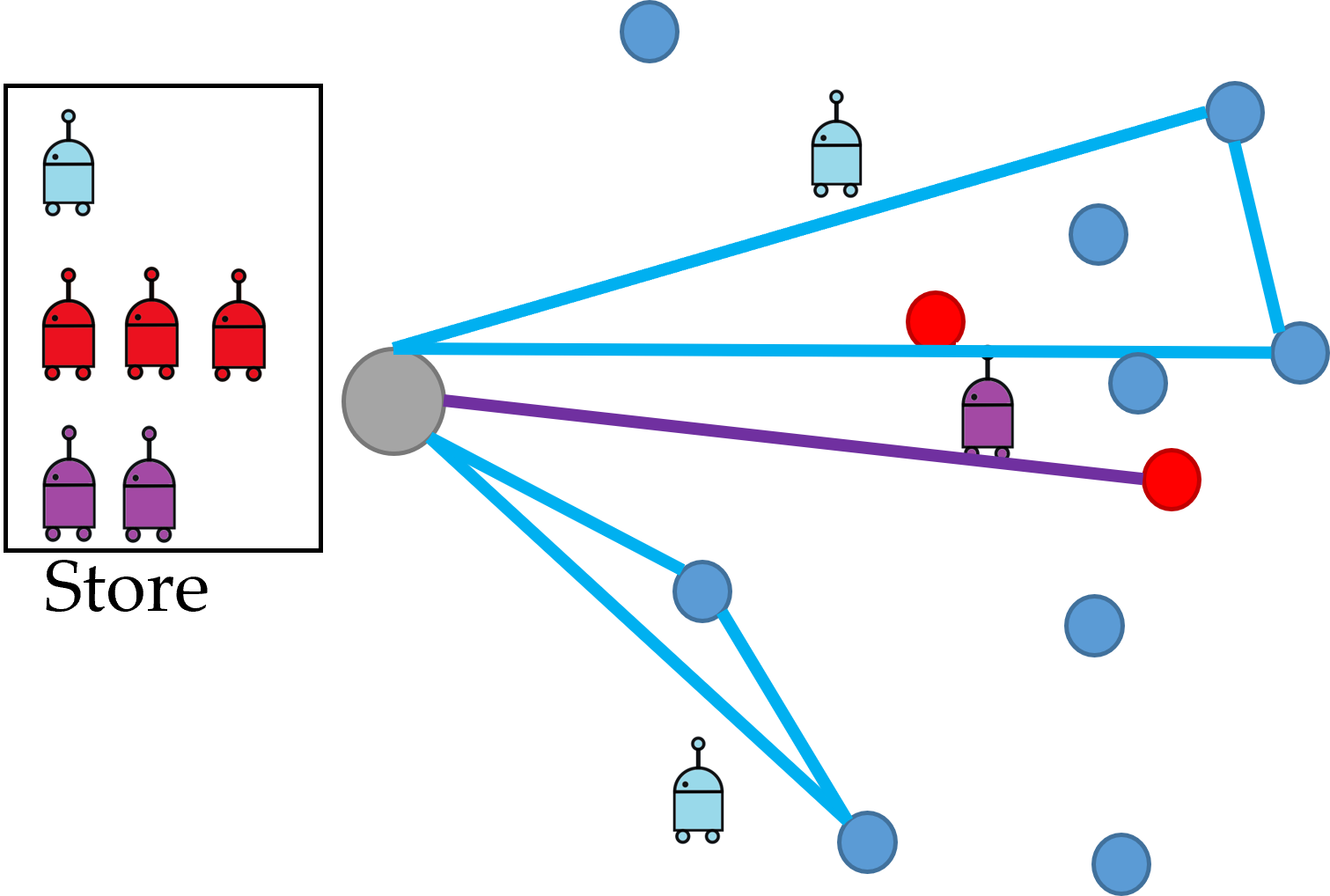}
    \caption{Iteration 3: $\delv$.}
        \label{fig:alg_iter3remove}
    \end{subfigure}
    \hspace{.5cm}
    \begin{subfigure}[t]{0.2\textwidth}
         \centering
        \includegraphics[width=0.8\linewidth]{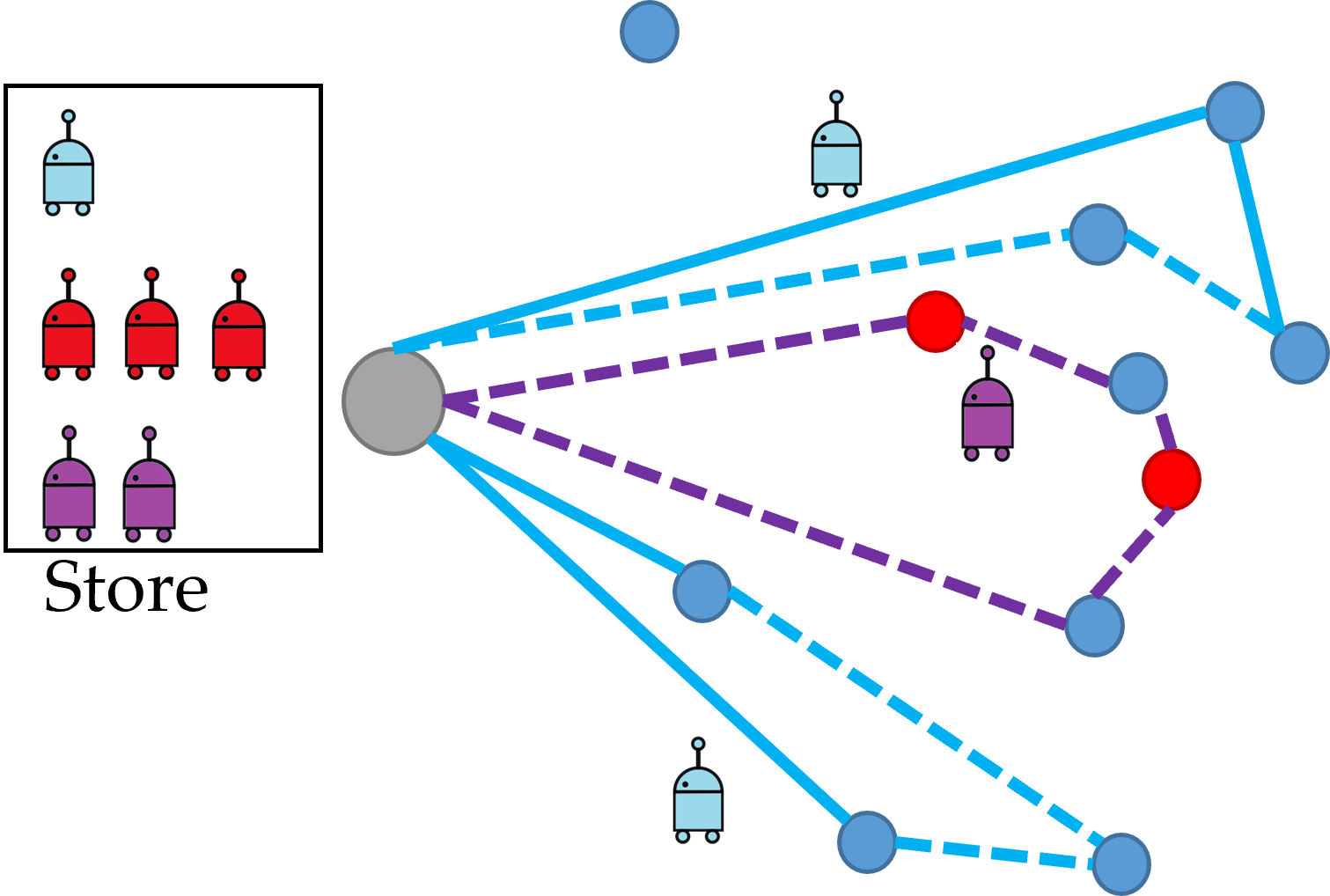}
    \caption{Iteration 3: $\mathtt{Repair}$.}
        \label{fig:alg_iter3repair}
    \end{subfigure}
    \caption{Illustration of Algorithm \ref{alg:FLNS} over three iterations. The grey dot is the depot, coloured dots are task locations. Blue and red robots can only service tasks of their respective colour, purple robots can service any task. Solid lines indicate current robot tours, dashed lines show edges added during $\mathtt{Repair}$. Robots in the `store' are currently not in the fleet and thus do not count towards the budget.
    \vspace{-.8cm}}
    \label{fig:algo}
\end{figure}
\paragraph{Example Illustration}
We provide an example of the LNS algorithm in Figure \ref{fig:algo}. Here the budget constraint allows for using at most three robots of any type.
An initial solution deploys two blue and one red robot, allowing the fleet to service 7 of 11 tasks.
In iteration $1$, a $\delv$ enables the upper blue robot to change its tour, increasing the number of serviced tasks to $8$ (subfigure b and c). 
This is the best attainable solution for the current fleet. In the second iteration, the algorithm performs a $\delr$ and subsequently replaces the red robot with a purple one, which is able to service yet another additional tasks (subfigure d and e).
Finally, by performing $\delv$ again, the purple robot takes over a task from the lower blue robot, allowing the latter to service another previously neglected task. Thus, in the final solution 10 out of 11 tasks are being serviced.

\section{Simulations}

We evaluate our work in a series of simulation experiments, comparing against two baselines.
\subsection{Experiment Setup}
\paragraph{Scenarios}
We consider three scenarios: \textit{Experiment 1} has a simple setup with only three ground robots and two task types. 
In \textit{Experiment 2} tasks have no deadlines, features more task requirements and offers a larger variety of robots.
In \textit{Experiment 3} we consider ground robots and aerial robots that jointly have to take different measurements, yet not all measurements can be taken by either robot type. Further, ground robots have to avoid obstacles while aerial robots can travel directly between task locations.
\paragraph{Baselines}
We consider two baselines: Selecting robots randomly until the budget is exhausted ($\mathtt{Random}$), and the greedy approach ($\mathtt{Greedy}$) described in Section \ref{sec:greedy}. To solve the underlying MRTA problem for computing \eqref{eq:greedy_step} and the final set of tours, we use another large neighbourhood search that optimizes the tours for a fixed fleet.
Using the MILP as a baseline is impractical due to very high runtime for even small instances.
Our approach is labelled as $\FLNS$.

\paragraph{Algorithm parameter}
The algorithm parameters are set up as follows:
Iteration budget $K=1000$, bias for random selection of removal heuristic $p^{\mathtt{removal}}=\nicefrac{1}{3}$, max.~$\%$ of robots removed by $\delr$ $n^{\mathtt{R}}=25$, max.~$\%$ of tasks removed by $\delv$ $n^{\mathtt{T}}=50$, and bias to neglect deployment cost in $\mathtt{Repair}$ $p^{\mathtt{discount}} = \nicefrac{1}{10}$.

\paragraph{Environment}
We use a schematic real-world campus map, shown in Figure \ref{fig:example_experiment}. Given {uniformly sampled task locations}, we  create a complete meta-graph, where vertices correspond only to the depot and task locations, and edge lengths are given by the shortest paths on a PRM.
All task deadlines are $t^d=150$. We repeat experiments for 20 trials.

\subsection{Results}
\paragraph{Experiment 1}
In the first experiment, we consider two different task types and only three robot types with varying capabilities and cost, summarized in Table \ref{tab:robot_types1}. 

\begin{table}[!t]

\begin{subtable}{.49\textwidth}
\centering
    \begin{tabular}{l c cccc}
        \toprule

       ID & capabilities & speed & battery&cost & type
        \\
        \midrule         
        $1$ &   $\{1\}$ & $100\%$ & $200$ & $20$& AGV \\
        $2$ &   $\{2\}$ & $100\%$ & $200$ & $20$ & AGV\\
        $3$ &   $\{1,2\}$ & $150\%$ & $500$ & $25$ & AGV\\
       
        \bottomrule
    \end{tabular}
    \caption{Experiment 1.}
    \label{tab:robot_types1}
\end{subtable}
    
    \begin{subtable}{.49\textwidth}
    \centering
\begin{tabular}{l c cccc}
        \toprule
       ID & capabilities & speed & battery&cost & type
        \\
        \midrule         
        $1$ &   $\{1\}$ & $100\%$ & $300$ & $20$ & AGV\\
        $2$ &   $\{2\}$ & $100\%$ & $300$ & $20$& AGV \\
        $3$ &   $\{3\}$ & $100\%$ & $300$ & $20$ & AGV\\
        $4$ &   $\{1,2,3\}$ & $150\%$ & $300$ & $30$& AGV \\
        $5$ &   $\{1,2\}$ & $100\%$ & $250$ & $25$& AGV \\       
        \bottomrule
    \end{tabular}
    \caption{Experiment 2.}
    \label{tab:robot_types2}
    \end{subtable}

    \begin{subtable}{.49\textwidth}
    \centering
    \begin{tabular}{l c ccc c}
        \toprule
       ID & capabilities & speed & battery&cost  & type
        \\
        \midrule         
        $1$ &   $\{1\}$ & $100\%$ & $300$ & $20$ & AGV \\
        $2$ &   $\{1, 2\}$ & $150\%$ & $300$ & $25$ & AGV\\
        $3$ &   $\{1,3\}$ & $200\%$ & $300$ & $20$ & UAV\\
        $4$ &   $\{3\}$ & $200\%$ & $250$ & $10$& UAV \\
        $5$ &   $\{1\}$ & $300\%$ & $250$ & $15$& UAV \\       
        \bottomrule
    \end{tabular}
    \caption{Experiment 3.}
    \label{tab:robot_types3}
    \end{subtable}
    \caption{Different types of robots used in the Experiments.
    \vspace{-.5cm}
    }   
\end{table}

Figure \ref{fig:example_experiment} illustrates example solutions for $\Greedy$ and $\FLNS$ for $N=60$ tasks and a budget of $B=70$. $\Greedy$ iteratively adds one robot with ID $1$, $2$, and $3$, resulting in a fleet that services $27$ tasks before their deadlines. The $\FLNS$ approach assembles a fleet with one robot of type $2$ and two robots of type $3$, allowing the fleet to service $38$ tasks.
The example highlights the main shortcoming of a greedy approach: In early iterations, low budget robots have a large marginal gain and are thus selected. Yet, in later iterations the  remaining budget does not suffice to add multiple flexible robots that could combine different task types.
In contrast, $\FLNS$ employs multiple of the more expensive but also much more capable type 3 robots, leading to more tasks being serviced.
\begin{figure}[t]
    \centering
    \begin{subfigure}[t]{0.2\textwidth}
         \centering
        \includegraphics[width=0.99\linewidth]{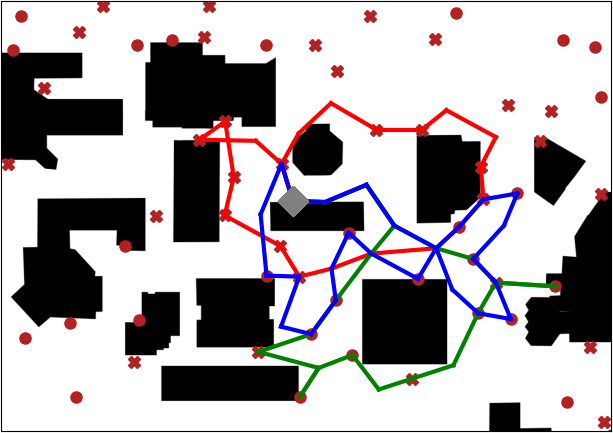}
    \caption{Greedy solution.}
        \label{fig:example_experiment_greedy}
    \end{subfigure}
    \quad
    \begin{subfigure}[t]{0.2\textwidth}
         \centering
        \includegraphics[width=0.99\linewidth]{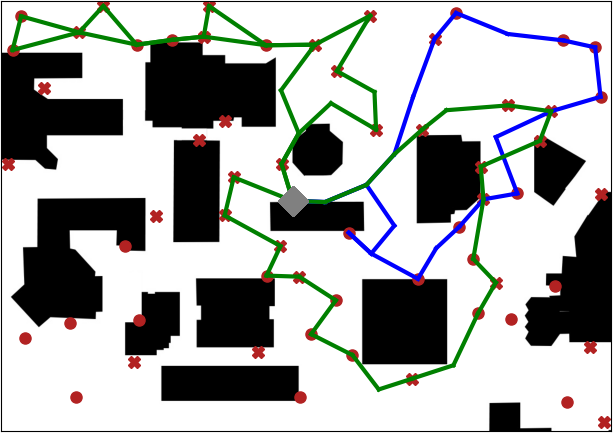}
    \caption{LNS solution.}
        \label{fig:example_experiment_LNS}
    \end{subfigure}
    \caption{Example for Experiment 1. Different colors indicate different robot types from Table \ref{tab:robot_types1} (ID 1 -- red, ID 2 -- blue, ID 3 -- green)
    \vspace{-.5cm}
    }
    \label{fig:example_experiment}
\end{figure}

We repeat the experiment for varying budgets and numbers of tasks, and quantitative results are shown in Figure \ref{fig:results1}.
With increasing budgets, all methods are able to service more tasks. However, while $\Greedy$ performs only slightly better than $\Random$, the proposed method services significantly more tasks under almost all settings. In particular, $\FLNS$ only requires roughly half of the budget ($B=50$) to achieve the same performance as the baselines with the full budget ($B=100$). For $N=20$, $\FLNS$ is able to service all tasks as the budget increases. Overall, the relative performance of the baselines becomes comparably poorer for larger $N$ even for the full budget $B=100$: For $N=20$ $\Greedy$ and $\FLNS$ service $100\%$ of tasks. For $N=60$, $\FLNS$ still achieves $100\%$, yet $\Greedy$ falls to $83\%$ and for $N=100$ they achieve $90\%$ and $70\%$, respectively.
In summary, in the simple setup with only few robot types the proposed method is able to find substantially stronger solutions than a greedy approach.

\begin{figure}[t]
\centering
\begin{subfigure}[t]{0.45\textwidth}
\centering
        \includegraphics[width=.9\linewidth]{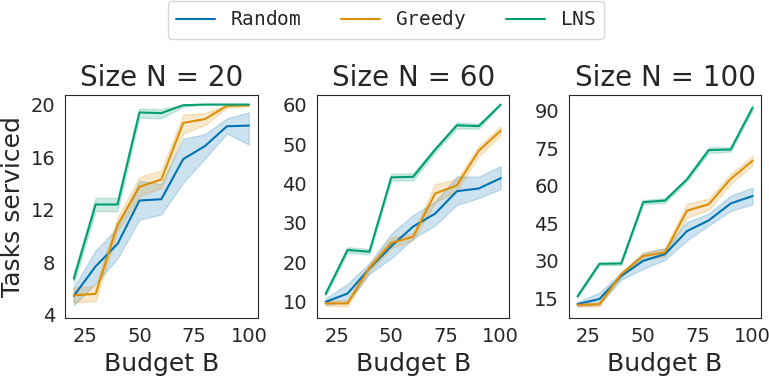}   
    \caption{Experiment 1.}
    \label{fig:results1}
\end{subfigure}\\
\begin{subfigure}[t]{0.45\textwidth}
\centering
        \includegraphics[width=0.9\linewidth]{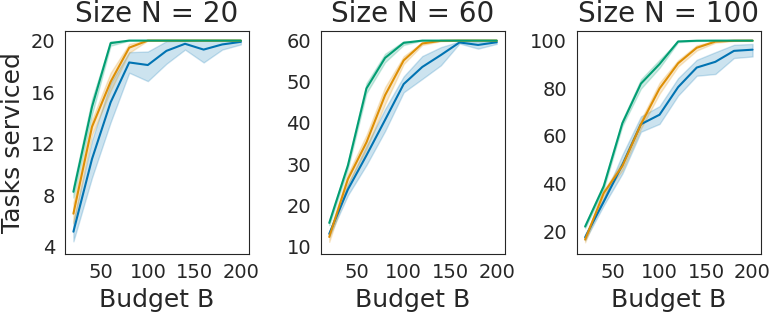}   
    \caption{Experiment 2. }
    \label{fig:results2}
\end{subfigure}\\
\begin{subfigure}[t]{0.45\textwidth}
\centering
        \includegraphics[width=0.9\linewidth]{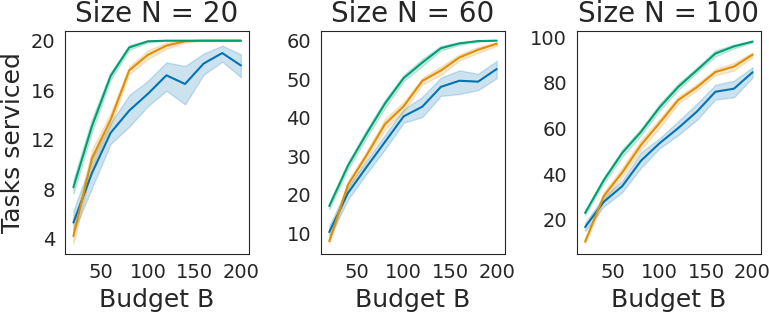}   
    \caption{Experiment 3.}
    \label{fig:results3}
    \end{subfigure}
    \caption{Experimental Results: \# serviced tasks for different budgets.
    \vspace{-.8cm}
    }
\end{figure}
\paragraph{Experiment 2}
Experiment 2 extends the setup to three task requirements and five available robots (see Table \ref{tab:robot_types2}), and removes task deadlines.
Figure \ref{fig:results2} shows a similar trend as in the first experiment, yet the larger budgets allow all for more tasks being completed. $\FLNS$ requires a substantially smaller budget to match the best performance of $\Greedy$: To achieve $\approx100\%$ serviced tasks under different values for $N= 20, 60,  100$, $\FLNS$ needs a budget of $B=80$, $B=120$, and $120$, while $\Greedy$ requires $B=100$, $B=140$, and $180$, respectively.

\paragraph{Experiment 3}
The third experiment considers a data collection mission where ground vehicles and drones take measurements. There are four different task types, of which some are exclusive to different robot types, as listed in Table \ref{tab:robot_types3}.
Similar to the other experiments, Figure \ref{fig:results3} shows that $\FLNS$ consistently outperforms $\Greedy$. For the different values of $N$, $\Greedy$ achieves its highest performance for $B=160$, $B=200$ and $B=200$, respectively. $\FLNS$ requires only budgets of $B=100$, $B=160$ and $B=160$, respectively, to achieve the same number of serviced tasks. As in Experiment 2, the performance gap is largest for relatively small budgets. This suggests that the performance of $\Greedy$ suffers from suboptimal choices in early iterations. The $\FLNS$ approach is able to avoid these local optima and thus find better solutions.

\paragraph{Runtime}
Overall, $\Greedy$ and $\FLNS$ perform comparably for small problems ($N=20$), running on average within $15s$ per instance. In Experiment 1, $\Greedy$ performs better for large instances ($N=100$), resulting in $ 150s$ compared to $200s$ for $\FLNS$. However, Experiment 2 and 3 feature more robots, causing the runtime of $\Greedy$ to increase drastically to $>500s$ for the larger instances. In contrast, $\FLNS$ maintains an average runtime of $200s$. Moreover, using only $100$ iterations for $\FLNS$ reduces its runtime by factor $10$, yet $\FLNS$ still substantially outperforms $\Greedy$ in Experiments 1 and 2, and by a small margin in Experiment 3.%
%
%
%
\section{Discussion and Future Work}

We studied the problem of designing heterogeneous robot fleets for multi-robot task assignment. We provided a MILP formulation and presented an LNS algorithm. In simulation experiments, we demonstrated that the LNS approach consistently finds better solutions than a greedy approach, \textit{i.e.,} requires smaller budgets to service the same number of tasks, while its runtime scales better to large instances.

A limitation of our work is the simple cost model for robot deployment. This could be extended to operation costs that increases with the robot's continued deployment. Further, our experiment relied on synthetic data. Using real-world robot models and realistic travel capabilities would highlight the practical benefits of the proposed method.
Lastly, we considered tasks that are serviced independently, \textit{i.e.,} one robot's mission does not affect how another robot can service its tasks. Yet, in some applications such as search-and-rescue, one robot's task completion might benefit other robots. This poses further challenges to the fleet design problem.

\bibliographystyle{IEEEtran}

\end{document}